\documentclass{article}

\usepackage{microtype}
\usepackage{graphicx}
\usepackage{booktabs} 
\usepackage{hyperref}
\usepackage{amsthm}
\usepackage{mathtools}
\usepackage{amssymb}
\usepackage{comment}

\usepackage{cleveref}

\usepackage{hyperref}

\usepackage[accepted]{icml2018}

\usepackage{varwidth}
\usepackage[T1]{fontenc}
\usepackage[utf8]{inputenc}
\usepackage{algorithm}
\usepackage{algorithmic}
\usepackage{blindtext, subfig}
\usepackage{dblfloatfix} 
\usepackage[font=footnotesize]{caption}

\makeatletter
\newcommand*\bigcdot{\mathpalette\bigcdot@{.5}}
\newcommand*\bigcdot@[2]{\mathbin{\vcenter{\hbox{\scalebox{#2}{$\m@th#1\bullet$}}}}}
\makeatother

\icmltitlerunning{Beyond the One-Step Greedy Approach in Reinforcement Learning}

\makeatletter
\def\BState{\State\hskip-\ALG@thistlm}
\makeatother

\newtheorem{proposition}{Proposition}
\newtheorem{corollary}[proposition]{Corollary}
\newtheorem{lemma}[proposition]{Lemma}
\newtheorem{theorem}[proposition]{Theorem}

\newtheorem{remark}{Remark}

\def\eqdef{\stackrel{\text{def}}{=}}

\def\E{{\mathbb E}}

\newcommand\Ga[2]{{\hat {\cal G}^{#1}_{#2}}}
\def\G{{\mathcal G}}
\def\S{{\mathcal S}}
\def\A{{\mathcal A}}

\newcommand{\condE}[1]{\E_{\mid #1}}

\begin{document}
\twocolumn[
\icmltitle{Beyond the One-Step Greedy Approach in Reinforcement Learning}




\begin{icmlauthorlist}
\icmlauthor{Yonathan Efroni}{tech}
\icmlauthor{Gal Dalal}{tech}
\icmlauthor{Bruno Scherrer}{inr}
\icmlauthor{Shie Mannor}{tech}
\end{icmlauthorlist}

\icmlaffiliation{tech}{Technion, Israel Institute of Technology}
\icmlaffiliation{inr}{INRIA, Villers-lès-Nancy, F-54600, France}

\icmlcorrespondingauthor{Yonathan Efroni}{jonathan.efroni@gmail.com}

\icmlkeywords{Machine Learning, ICML}

\vskip 0.3in
]



\printAffiliationsAndNotice{} 

\begin{abstract}

The famous Policy Iteration algorithm alternates between policy improvement and policy evaluation. Implementations of this algorithm with several variants of the latter evaluation stage, e.g, $n$-step and trace-based returns, have been analyzed in previous works. However, the case of multiple-step lookahead policy improvement, despite the recent increase in empirical evidence of its strength, has to our knowledge not been carefully analyzed yet. In this work, we introduce the first such analysis. Namely, we formulate variants of multiple-step policy improvement, derive new algorithms using these definitions and prove their convergence. Moreover, we show that recent prominent Reinforcement Learning algorithms fit well into our unified framework. We thus shed light on their empirical success and give a recipe for deriving new algorithms for future study.
\end{abstract}

\section{Introduction}
Policy Iteration (PI) lies at the core of Reinforcement Learning (RL) and of many planning and on-line learning methods \cite{konda1999actor,kakade:02,schulman2015trust,mnih2016asynchronous,silver2017mastering}. 
The classic PI algorithm repeats consecutive stages of i) 1-step greedy policy improvement with respect to (w.r.t.) a value function estimate, and ii) evaluation of the value function w.r.t. the greedy policy. Although multiple variants of the evaluation task have been considered \cite{puterman1978modified, lpi,sutton1998reinforcement}, the usually considered policy update is the 1-step greedy improvement.    

Conducting policy improvement using the common 1-step greedy approach is a specific choice, which is not necessarily the most appropriate one. Indeed, it was empirically recently suggested that greedy approaches w.r.t. multiple steps perform better than w.r.t. 1-step. Notable examples are Alpha-Go and Alpha-Go-Zero \cite{silver2016mastering, silver2017mastering,silver2017mastering2}. There, an approximate online version of multiple-step greedy improvement is implemented via Monte Carlo Tree Search (MCTS) \cite{browne2012survey}. 
The celebrated MCTS algorithm, which instantiates several steps of lookahead improvement, encompasses additional historical impressive accomplishments dating back to the past century and previous decade \cite{tesauro1997line,sheppard2002world,bouzy2004monte,veness2009bootstrapping}. 
To the best of our knowledge, and despite such empirical successes, the use of a multiple-step greedy policy improvement has never been rigorously studied before. The motivation of this work is to fill this gap, and suggest new possible algorithms in this spirit. 

The paper is organized as follows. We start by defining the $h$-greedy policy, 
a policy that is greedy w.r.t. a horizon of $h$ steps. Using this definition, we introduce the $h$-PI algorithm, a class of PI algorithms with multiple-step greedy policy improvement and a guarenteed convergence to the optimal policy. We stress that the term ``$n$-step return'' often refered to in the literature, is used in the context of policy evaluation \cite{sutton1998reinforcement,seijen2014true}. Therefore, to avoid confusion, we choose to denote the multiple-step greedy policy by the letter $h$.

 We then introduce a novel class of optimal Bellman operators, which is controlled by a continuous parameter $\kappa\in[0,1]$. This operator is used to define a new greedy policy, the $\kappa$-greedy policy, leading to a new PI algorithm which we name $\kappa$-PI. 
This is analogous to the famous TD($\lambda$) algorithm \cite{sutton1988learning} for the improvement stage. In the TD($\lambda$) algorithm, the $\lambda$ parameter interpolates between the single-step evaluation update for $\lambda =0$ and the infinite-horizon (Monte Carlo) evaluation update for $\lambda=1$. Similarly, for $\kappa=0$, we recover the traditional 1-step greedy policy for the improvement update and for $\kappa=1$ we get the infinite-horizon greedy policy, i.e. the optimal policy. Roughly speaking, the $\kappa$-greedy policy can be viewed as allowing to make an `interpolation' over all $h$-greedy policies. Similarly to the previous paragraph, we use the letter $\kappa$ to avoid confusion with the parameter $\lambda$ of TD($\lambda$). Remarkably, we show that computing the $\kappa$-greedy policy is equivalent to solving the optimal policy of a surrogate $\kappa\gamma$-discounted and stationary MDP.

As an additional generalization, we introduce the $\kappa\lambda$-PI. This algorithm has a similar improvement step as the $\kappa$-PI, but its policy evaluation stage is `relaxed', similarly to $\lambda$-PI \cite{lpi}.
The $\kappa\lambda$-PI further illustrates the difference between the $\lambda$ and $\kappa$ parameters. While the former controls the depth of the evaluation task in a way similar to previous works \cite{sutton1998reinforcement,seijen2014true}, the latter controls the depth of the improvement step.  


Next, we discuss the relation of this work to existing literature, and argue that it offers a generalized view for several recent impressive empirical advancements in RL, which are seemingly unrelated \cite{schulman2015high,silver2017mastering}. We thus show relevance of our proposed mathematical framework to current state-of-the-art algorithms. We conclude with an empirical display of the influence of these new parameters $\kappa$ and $h$ on a basic planning task. We empirically demonstrate that the best performance is obtained with non-trivial choices of them. This motivates future study of new RL algorithms which can be derived from the introduced framework in this work.
\section{Preliminaries}
Our framework is the infinite-horizon discounted Markov Decision Process (MDP). An MDP is defined as the 5-tuple $(\mathcal{S}, \mathcal{A},P,R,\gamma)$ \cite{puterman1994markov}, where ${\mathcal S}$ is a finite state space, ${\mathcal A}$ is a finite action space, $P \equiv P(s'|s,a)$ is a transition kernel, $R \equiv r(s,a)$ is a reward function, and $\gamma\in(0,1)$ is a discount factor. Let $\pi: \mathcal{S}\rightarrow \mathcal{P}(\mathcal{A})$ be a stationary policy, where $\mathcal{P}(\mathcal{A})$ is a probability distribution on $\mathcal{A}$. Let $v^\pi \in \mathbb{R}^{|\mathcal{S}|}$ be the value of a policy $\pi,$ defined in state $s$ as $v^\pi(s) \equiv \condE{s}^\pi[\sum_{t=0}^\infty\gamma^tr(s_t,\pi(s_t))]$, where $\condE{s}^\pi$ denotes expectation w.r.t. the distribution induced by $\pi$ and conditioned on the event $\{s_0=s\}.$  For brevity, we respectively denote the reward and value at time $t$ by $r_t\equiv r(s_t,\pi_t(s_t))$ and $v_t\equiv v(s_t).$  It is known that
\begin{align*}
v^\pi=\sum_{t=0}^\infty \gamma^t (P^\pi)^t r^\pi=(I-\gamma P^\pi)^{-1}r^\pi,
\end{align*}
with the component-wise values $[P^\pi]_{s,s'}  \triangleq P(s'\mid s, \pi(s))$ and $[r^\pi]_s \triangleq  r(s,\pi(s))$. Our goal is to find a policy $\pi^*$ yielding the optimal value $v^*$ such that
\begin{equation}
v^* = \max_\pi (I-\gamma P^\pi)^{-1} r^\pi = (I-\gamma P^{\pi^*})^{-1} r^{\pi^*}.\label{mdp}
\end{equation}
This goal can be achieved using the three classical operators (with equalities holding component-wise):  
\begin{align}
\forall v,\pi,~  T^\pi v & =  r^\pi +\gamma P^\pi v, \label{def: Tpi} \\
\forall v,~  T v & =  \max_\pi T^\pi v, \\
\forall v,~\G(v)&= \{\pi : T^\pi v = T v\}, \label{def: greedy policy}
\end{align}
where $T^\pi$ is a linear operator, $T$ is the optimal Bellman operator and both $T^\pi$ and $T$ are $\gamma$-contraction mappings w.r.t. the max norm. It is known that the unique fixed points of $T^\pi$ and $T$ are $v^\pi$ and $v^*$, respectively. The set $\G(v)$ is the standard set of 1-step greedy policies w.r.t. $v$. 
 Furthermore, given $v^*$, the set $\G(v^*)$ coincides with that of stationary optimal policies. In other words, every policy that is 1-step greedy w.r.t. $v^*$ is optimal and vice versa.

\section{The $h$-Greedy Policy and $h$-PI}
\label{sec: h-step greedy}

In this section we introduce the $h$-greedy policy, a generalization of the 1-step greedy policy. This leads us to formulate a new PI algorithm which we name ``$h$-PI''. The $h$-PI is derived by replacing the improvement stage of the PI, i.e, the 1-step greedy policy, with the $h$-greedy policy. We further prove its convergence and show it inherits most properties of PI.


Let $h\in \mathbb{N}\textbackslash \{0\}$. A $h$-greedy policy w.r.t. a value function $v$ belongs to the following set of policies,
\begin{align}
  & \arg\max\limits_{\pi_0}\! \max\limits_{\pi_1,..,\pi_{h-1}} \!\!\!\condE{\bigcdot}^{\pi_0\dots\pi_{h-1}}\!\!\left[\sum_{t=0}^{h-1}\!\!\gamma^t r(s_t,\!\pi_t(s_t))\!+\!\gamma^h v(s_h)\right] \nonumber\\
  &\quad=\arg\max_{\pi_0} \condE{\bigcdot}^{\pi_0}\left[r(s_0,\pi_0(s_0))\!+\!\gamma \!\left(T^{h-1} v\right)\!(s_1) \right]\label{eq_h_greedy_def}
\end{align}
where the notation $\condE{\bigcdot}^{\pi_0\dots\pi_{h-1}}$ means that we condition on the trajectory induced by the choice of actions $(\pi_0(s_0),\pi_1(s_1),\dots, \pi_{h-1}(s_{h-1}))$ and the starting state $s_0=\bigcdot$. Since the argument in \eqref{eq_h_greedy_def} is a vector, the maximization is component-wise, i.e., we wish to find the choice of actions that will jointly maximize the entries of the vector.
Thus, the $h$-greedy policy chooses the first 
optimal action of a non-stationary, optimal control problem with horizon $h$. 
As the equality in \eqref{eq_h_greedy_def} suggests, this policy can be equivalently interpreted as the 1-step greedy policy w.r.t. $T^{h-1}v$. 
Although the former view is more natural, it is, in fact, the latter on which this section's proofs are based.  
Thus, the set of $h$-greedy policies w.r.t. $v$, $\G_h(v)$, can be expressed as follows:
\begin{align}
\forall v,\pi,~ T_h^\pi v &\eqdef T^{\pi}T^{h-1} v, \label{eq: Thpi eval def} \\
\forall v,~\G_h(v)&= \{\pi: T_h^\pi v = T^h v\}. \label{eq: Thpi def}
\end{align}
\begin{remark}\label{astar}
	This is a generalization of the standard 1-step greedy operation, which one recovers by taking $h=1$.
	Each call to the greedy operator $\G_h$ amounts to identifying for all 
	 states, the best first action of an $h$-horizon optimal control problem.
	More interestingly, one may compute these first optimal actions with specifically designed procedures such as $A^*$-like/optimistic tree exploration algorithm \cite{hren08, busoniu12, munosbook14, szorenyi14, grill16}.
\end{remark}

With this set of operators, we consider Algorithm~\ref{alg:hPI}, the $h$-PI algorithm, where the assignments hold point-wise. 
This algorithm alternates between i) identifying the $h$-greedy policy, i.e, solving the $h$-horizon optimal control problem, and ii) estimating the value of this policy.
%
%
\begin{algorithm}[H]
	\caption{$h$-PI}
	\label{alg:hPI}
	\begin{algorithmic}
		\STATE {\bfseries Initialize:} $h \in \mathbb{N} \textbackslash \{0\},~v \in \mathbb{R}^{|\S|}$
		\WHILE{the value $v$ changes}
		\STATE $\pi \gets \arg\max\limits_{\pi_0}\!\max\limits_{\pi_1,..,\pi_{h-1}}\!\!\!\condE{\bigcdot}^{\pi_0\dots\pi_{h-1}}\!\!\left[\sum_{t=0}^{h-1}\gamma^t r_t+\gamma^h v_h\right]$
		\STATE $v \gets \condE{\bigcdot}^{\pi}\left[\sum_{t=0}^\infty \gamma^t r_t\right]$
		\ENDWHILE
		\STATE {\bfseries Return $\pi,v$}
	\end{algorithmic}
\end{algorithm}

As we are about to see, this new algorithm inherits most properties of standard PI. We start by showing a monotonicity property for the $h$-greedy operator.
\begin{lemma}[Policy improvement of the $h$-greedy policy]\label{lemma_hPI_improves}
	Let $\pi' \in \G_h(v^\pi).$ Then $v^{\pi'}\geq v^{\pi}$ component-wise, with equality holding if and only if $\pi$ is an optimal policy.
\end{lemma}
\begin{proof}
First observe that
\begin{equation}
\label{eq_policy_improve_1st_relation}
v^\pi = T^{\pi} v^\pi \leq T v^\pi.
\end{equation}
Then, sequentially using \eqref{eq_policy_improve_1st_relation}, \eqref{eq: Thpi def} and \eqref{eq: Thpi eval def}, we have
\begin{align}
v^\pi  = (T^\pi)^h v^\pi \leq T^h v^\pi = T^{\pi'}_h v^\pi &= T^{\pi'}(T^{h-1}v^\pi) \label{eq_policy_improve}.
\end{align}
This leads to the following inequalities:
\begin{align}
v^\pi&\leq T^{\pi'}(T^{h-1}v^\pi) \label{eq: 1st inequality}\\
&\leq T^{\pi'}(T^{h-1}Tv^\pi)= T^{\pi'}(T^{h}v^\pi) \label{eq: 2nd inequality}\\
& = T^{\pi'}\left(T^{\pi'}T^{h-1}v^\pi\right)= \left(T^{\pi'}\right)^2(T^{h-1}v^\pi)  \label{eq: 3rd inequality}\\
&\leq \cdots \leq \lim_{n\rightarrow \infty}\left(T^{\pi'}\right)^n(T^{h-1} v^{\pi}) = v^{\pi'}. \label{eq: VI}
\end{align}
In the above derivation, \eqref{eq: 1st inequality} is due to \eqref{eq_policy_improve}, \eqref{eq: 2nd inequality} is due to \eqref{eq_policy_improve_1st_relation} and the monotonicity of $T^{\pi'}$ and $T$ (and thus of their composition), \eqref{eq: 3rd inequality} is due to \eqref{eq_policy_improve}, and \eqref{eq: VI} is due to the fixed point property of $T^{\pi'}$. Lastly, notice that $v^\pi=v^{\pi'}$ if and only if (cf. \eqref{eq_policy_improve_1st_relation}) $T v^\pi = v^\pi$, which holds if and only if $\pi$ is the optimal policy.\end{proof} 

Thus, the improvement property of the 1-step greedy policy also holds for the $h$-greedy policy. As a consequence, the $h$-PI algorithm produces a sequence of policies with component-wise increasing values, which directly implies convergence (since the sequence is bounded). We can be more precise about the convergence speed by generalizing several known results on PI to $h$-PI. Let us begin by the following lemma, which is essentially a consequence of the fact that $T^h$ is a $\gamma^h$-contraction (see Appendix~\ref{proof_lemma_contracting_sequence_h_PI} in the supplementary material for a proof). 
\begin{lemma}\label{lemma_contracting_sequence_h_PI}
	The sequence $\left\{\|v^* - v^{\pi_k}\|_\infty\right\}_{k\geq0}$ is contracting with coefficient $\gamma^h$.
\end{lemma}

%
Thus, the convergence rate is at most $\gamma^h$, which generalizes the known $\gamma$ convergence rate of the original ($h=1$) PI algorithm \cite{puterman1994markov}. The next theorem describes a result with respect to the termination of the algorithm.
%
\begin{theorem}\label{complexity2}
	The $h$-PI algorithm converges in at most $|\S|(|\A|-1)\left\lceil (h \log{\frac 1\gamma})^{-1} \log\left(\frac{1}{1-\gamma} \right) \right\rceil $ iterations.
\end{theorem}
\begin{proof}
		The proof follows the steps of \cite{pi_complexity}, Section~7, where instead of using the contraction coefficient of PI algorithm we use the contraction coefficient of $h$-PI, proved in Lemma \ref{lemma_contracting_sequence_h_PI}. \end{proof}

\begin{remark}
	\label{rem: non-trivial extension}
	Note that the fact that $T^h$ is $\gamma^h$-contraction does not imply that all existing PI results extend to $h$-PI with $\gamma^h$ replacing $\gamma$. For instance, the rightmost term in the logarithm of Theorem~\ref{complexity2} is $\frac{1}{1-\gamma}$ and not $\frac{1}{1-\gamma^h}$.
\end{remark}

\begin{remark} \label{rem: h-dp complexity}
	Notice how the complexity term in Theorem~\ref{complexity2} is a decreasing function of $h$. At the limit $h\rightarrow\infty$, running a single iteration of $h$-PI is sufficient for finding the optimal value-policy pair. However, although the total number of iterations is reduced with $h$ increasing, each iteration is expected to be computationally more costly.
\end{remark}


\section{The $\kappa$-Greedy Policy}\label{sec: kappa-step greedy}

In this section, we introduce an additional, novel generalization of the 1-step greedy policy: \emph{the $\kappa$-greedy policy}. Similarly to the previous section, the newly introduced greedy policy leads to a new PI algorithm, which we shall name ``$\kappa$-PI''. 
This generalization will be  based on the definition of a new \emph{$\kappa$-optimal Bellman operator}. This operator will also naturally lead to a new Value Iteration (VI) algorithm, ``$\kappa$-VI''. In the later Section~\ref{interpretation_and_relation}, we shall highlight the relation between the $\kappa$-greedy and the previously introduced $h$-greedy policy. 

\subsection{$\kappa$-Optimal Bellman Operator and the $\kappa$-Greedy Policy}
In this section we will derive the $\kappa$-optimal Bellman operator and define its induced greedy policy, the $\kappa$-greedy policy.  

Given a parameter $\kappa \in [0,1]$, consider the following operator:
\begin{equation}
\forall v,\pi,~ T_\kappa^\pi v \eqdef (1-\kappa) \sum_{j=0}^\infty \kappa^j (T^\pi)^{j+1} v. \label{def1}
\end{equation}
This linear operator is identical to the one 
of the $\lambda$-PI algorithm \cite{lpi}. By simple linear algebra arguments (see e.g. \cite{alpi}), one can see that
\begin{align}
\forall v,\pi,~ T_\kappa^\pi v & =  (I-\kappa\gamma P^\pi)^{-1}(r^\pi+(1-\kappa)\gamma P^\pi v) \label{def2} \\
& =  v+ (I-\kappa\gamma P^\pi)^{-1}(T^\pi v - v) \label{def4}.
\end{align}
Comparing \eqref{mdp} and \eqref{def2}, we can interpret  \eqref{def2}, given a fixed $v$, as the value of policy $\pi$ in a \emph{surrogate} stationary MDP. This MDP has the same dynamics as the MDP we wish to solve, a $\kappa \gamma\leq \gamma$ discount factor and a reward function $\hat r^{\pi,v}$ given by
\begin{equation}
\label{eq: shaped reward}
\hat r^{\pi,v} \eqdef r^\pi + (1-\kappa)\gamma P^\pi v.
\end{equation}
Thus, this surrogate stationary MDP depends on both $v$ and $\kappa$. According to basic MDP theory, the surrogate MDP has an optimal value. We shall denote its optimal value by $T_\kappa v$ and refer to $T_\kappa$ as the ``$\kappa$-optimal Bellman operator''. Note that, from \eqref{mdp} and \eqref{def2}, we have
\begin{align}
T_\kappa v
= \max_\pi (I-\kappa\gamma P^\pi)^{-1} \hat r^{\pi,v}= \max_\pi T_\kappa^\pi v.\label{prop1}
\end{align}
The $\kappa$-optimal Bellman operator naturally induces a new set of greedy policies, the set of $\kappa$-greedy policies w.r.t. $v$, which we shall denote by $\G_\kappa(v)$, and define as follows:
\begin{equation}\label{def_kappa_greedy_operator}
\forall v,~\G_\kappa(v) = \{ \pi: T_\kappa^\pi v = T_\kappa v\}.
\end{equation}
\begin{remark}\label{solving_with_smaller_discount}
	The $\kappa$-optimal Bellman operator is a generalization of the optimal Bellman operator, which one recovers by taking $\kappa=0$; i.e., $T_{\kappa=0}=T$. Additionally, applying once $T_{\kappa=1}$ is equivalent to solving the original $\gamma$-discounted MDP; i.e, for any $v$, $v^*=T_{\kappa=1} v$.
For all other values $0<\kappa<1$, applying $T_\kappa$ amounts to solving a stationary MDP with reduced discount factor, the solution of which can be obtained using any generic planning, model-free or model-based RL algorithm. As was previously analyzed in \cite{petrik2009biasing,strehl2009reinforcement,jiang2015dependence} and reported \cite{franccois2015discount}, solving a MDP with a smaller discount factor is in general easier.
\end{remark}

In the next lemma we prove that both $T_\kappa^\pi$ and $T_\kappa$ are contractions with respective fixed points $v^\pi$ and $v^*$.

\begin{lemma}
	\label{gmdp}
	For any $\pi$, $T_\kappa^\pi$ and $T_\kappa$ are $\xi$-contraction mappings w.r.t. the max norm, where $\xi=\frac{(1-\kappa)\gamma}{1-\gamma\kappa} \in [0,\gamma]$, and have unique fixed points $v^\pi$ and $v^*$, respectively. Moreover, $\G_\kappa(v^*)=\G(v^*).$ 
\end{lemma}
\begin{proof}
	From \eqref{def2}, we see that for all $v$ and $w$,
	\begin{equation}
	T_\kappa^\pi v - T_\kappa^\pi w = (1-\kappa)(I-\kappa\gamma P^\pi)^{-1}\gamma P^\pi (v-w),
	\end{equation}
	which implies, by taking the max norm,  that  $T_\kappa^\pi$ is a $\xi$-contraction mapping.
	
	From \eqref{prop1}, we see that for all $v$ and $w$,
	\begin{equation}
	T_\kappa^{\pi^*_w} v - T_\kappa^{\pi^*_w} w \le T_\kappa v - T_\kappa w \le T_\kappa^{\pi^*_v} v - T_\kappa^{\pi^*_v} w,
	\end{equation}
	and this in turn implies, by again taking the max norm, that $T_\kappa$ is also a $\xi$-contraction mapping.
	
	Since both operators are contraction mappings, they have one and only one fixed point. To identify them, it is thus sufficient to show that the foreseen fixed points are indeed fixed points.  By \eqref{def4}, since $v^\pi=T^\pi v^\pi$, it is clear that $v^\pi=T_\kappa^\pi v^\pi$.
	Now, from \eqref{prop1} and \eqref{def4},
	\begin{align}
	T_\kappa v^* = \max_\pi T_\kappa^\pi v^*& = v^* + z\label{T_k_fixed_point}, \\ \mbox{with~}z&=\max_\pi (I-\kappa\gamma P^\pi)^{-1}(T^\pi v^* - v^*)\nonumber.
	\end{align}
	By the optimality of $v^*$, we know that for any $\pi$, $T^\pi v^* - v^* \le 0$ ; since the matrix $(I-\kappa\gamma P^\pi)^{-1}=\sum_{i=0}^\infty (\kappa\gamma P^{\pi})^i$ is only made of non-negative coefficients, it follows that $z \le 0$. Then, since for $\pi=\pi^*$, $T^{\pi} v^*=v^*$, we deduce that $z=0$ and, as a consequence, $T_\kappa v^*=v^*$.
	
        Lastly, we show that $\G_\kappa(v^*)=\G(v^*)$. Let $\pi\in \G_\kappa(v^*)$. Thus, $T_\kappa^\pi v^* = T_\kappa v^*=v^*$, where the second equality holds since $v^*$ is the fixed point of $T_\kappa$. From \eqref{T_k_fixed_point} we deduce that $0=z=T^\pi v^*-v^*$. Thus,  $T^\pi v^*=v^*=Tv^*$ and $\pi$ is in $\G(v^*)$.        To show the opposite direction, we assume that $\pi$ is in $\G(v^*)$. Thus, it holds that $z=T^\pi v^*-v^*=0$; hence, $T^{\pi}_\kappa v^* = v^*$, and indeed $\pi$ is in $\G_\kappa(v^*)$.\end{proof}

\subsection{Two New Algorithms: $\kappa$-PI and $\kappa$-VI}


In the previous subsection, we derived the $\kappa$-optimal Bellman operator, and defined its induced $\kappa$-greedy policy. These operators lead us to consider Algorithm~\ref{alg:kappaPI}, the $\kappa$-PI algorithm, where the assignments hold component-wise.
%
%

\begin{algorithm}[H]
	\caption{$\kappa$-PI}
	\label{alg:kappaPI}
	\begin{algorithmic}
		\STATE {\bfseries Initialize:} $\kappa \in [0,1],~v \in \mathbb{R}^{|\S|}$
		\WHILE{the value $v$ changes}
		\STATE  $\pi \gets\arg\max\limits_{\pi'}\condE{\bigcdot}^{\pi'}\!\!\left[\sum_{t=0}^{\infty}(\kappa\gamma)^t (r_t+\gamma(1-\kappa) v_{t+1})\right]$
		\STATE $v \gets \condE{\bigcdot}^{\pi}\! \left[\sum_{t=0}^\infty \gamma^t r_t\right]$
		\ENDWHILE
		\STATE {\bfseries Return $\pi,v$}
	\end{algorithmic}
\end{algorithm}

This algorithm repeats consecutive steps of i) identifying the $\kappa$-greedy policy, i.e, solving the optimal policy of a surrogate, stationary MDP, with a reduced $\gamma\kappa$ discount factor, and ii) estimating the value of this policy. As we shall see, the iterative process is guaranteed to converge to the optimal policy-value pair of the MDP we wish to solve.  

Similarly to Section \ref{sec: h-step greedy}, we shall now prove that the $\kappa$-PI algorithm inherits many properties from PI. We start by showing a monotonicity property for the  $\kappa$-greedy operator.
\begin{lemma}[Policy improvement of the $\kappa$-greedy policy]\label{lemma_kPI_improvement}
	Let $\pi' \in \G_\kappa(v^\pi).$ Then $v^{\pi'}\geq v^{\pi}$ component-wise, with equality if and only if $\pi$ is an optimal policy.
\end{lemma}
\begin{proof}Since $v^\pi$ is also the fixed point of $T_\kappa^\pi$, by Lemma \ref{gmdp}, and using the definition of $T_\kappa$ (see \ref{prop1}),
\begin{align}
v^\pi=T_\kappa^\pi v^\pi \leq T_\kappa v^\pi = T^{\pi'}_\kappa v^\pi \label{eq_improvement_step}
\end{align}
by choosing $\pi' \in \G_\kappa(v^\pi)$. Using the monotonicity of the operator $T^{\pi'}_\kappa$ and repeating \eqref{eq_improvement_step} we get,
\begin{align*}
v^\pi=T_\kappa^\pi v^\pi &\leq  T^{\pi'}_\kappa v^\pi \leq \dots \leq  \lim _{n\rightarrow \infty} \left( T^{\pi'}_\kappa \right)^n v^\pi= v^{\pi'}.
\end{align*}
The final equality holds since $T^{\pi'}_\kappa$ is a contraction mapping, and $v^{\pi'}$ is its fixed point, due to Lemma \ref{gmdp}. According to the same lemma, equality holds if and only if $v^\pi=v^*$.\end{proof}

Similarly to Lemma \ref{lemma_contracting_sequence_h_PI}, let us state a result that is a direct consequence of the fact that $T_\kappa$, that induces the $\kappa$-greedy policy, is a $\xi$ contraction (see Appendix~\ref{proof_lemma_contracting_sequence_kPI} in the supplementary material for a proof).
\begin{lemma}\label{lemma_contracting_sequence_kPI}
	The sequence $\left(\|v^* - v^{\pi_k}\|_\infty\right)_{k\geq0}$ is contracting with coefficient $\xi$.
\end{lemma} 
In the following theorem to the lemma we upper bound the maximal number of iteration it takes the $\kappa$-PI algorithm to terminate  (the proof is the same as that for Theorem~\ref{complexity2}).

\begin{theorem}\label{complexity}
	The $\kappa$-PI algorithm converges in at most $|\S|(|\A|-1)\left\lceil (\log{\frac{1-\kappa\gamma}{(1-\kappa)\gamma}})^{-1} \log\left(\frac{1}{1-\gamma} \right) \right\rceil $ iterations.
\end{theorem}


\begin{remark}
	 As in the case of $h$-PI, the complexity term in Theorem~\ref{complexity} is a decreasing function of $\kappa.$ This complements the fact that $\kappa$-PI converges in one iteration if $\kappa=1$. Again, as $\kappa$ increases, each iteration is expected to be computationally more demanding since the surrogate MDP is less discounted. 
\end{remark}

Lastly, using the $\kappa$-optimal Bellman operators, $T_\kappa$, we derive a non-trivial generalization of the VI algorithm which we name ``$\kappa$-VI''. The $\kappa$-VI algorithm repeatedly applies $T_\kappa$ until convergence. The convergence proof of $\kappa$-VI and its rate of convergence, $\xi$, thus follows easily as a corollary of Lemma \ref{gmdp}. In the next section, we shall group both $\kappa$-PI $\kappa$-VI into a single, larger, class of algorithms that contains them both.

\section{$\kappa\lambda$-PI}
\label{sec: kappa_lambda_PI}

Even though $\kappa$-PI and $\kappa$-VI are distinct algorithms, in this section we unite them under the single ``$\kappa\lambda$-PI'' class of algorithms (Algorithm \ref{alg:klPI}). This generalization is similar to the $\lambda$-PI which interpolates between standard PI and VI \cite{lpi,bertsekas1995neuro,scherrer2013performance,bertsekas2015lambda}. We shall describe how to estimate the value of the $\kappa$-greedy policy on a surrogate MDP with a smaller horizon and shaped reward, and show that this still yields convergence to the optimal policy and value. Thus, in $\kappa\lambda$-PI, we ease the policy evaluation phase of $\kappa$-PI via solving a simpler task.

In the $\lambda$-PI, the improvement stage is the common $1$-step greedy policy, and the evaluation stage is relaxed by applying the $T_\lambda^\pi$ operator \eqref{def1} instead of fully estimating the value, where $\lambda\in[0,1]$. We start by formulating the appropriate generalization of the $T_\lambda^\pi$ operator to $\kappa$-PI. Let $\bar{\lambda}\in[0,1]$ . The analogous $T_\lambda^\pi$ to our framework is the following linear operator:
\begin{align*}
\forall v,\pi,\ T^{\pi}_{\bar{\lambda},\kappa}v = (1-\bar{\lambda})\sum_{j=0}^{\infty} \bar{\lambda}^j T_\kappa^\pi v.
\end{align*}
Interestingly, due to the fact this operator is affine, the following lemma shows that $T^{\pi}_{\bar{\lambda},\kappa}$ is equivalent to yet another $T_\lambda^\pi$ operator \eqref{def1} where $\lambda$ is a function of $\bar{\lambda}$ and $\kappa$ (see Appendix~\ref{proof_lemma_T_pi_kappa_lambda_is_T_lambda} in the supplementary material for a proof).
\begin{lemma}\label{lemma_T_pi_kappa_lambda_is_T_lambda}
For every $\bar{\lambda},\kappa\in[0,1]$, $T^{\pi}_{\bar{\lambda},\kappa}=T^{\pi}_\lambda$, where $\lambda= \kappa + \bar{\lambda}(1-\kappa) $, i.e, $\lambda\in[\kappa,1]$.
\end{lemma}


\begin{algorithm}[H]
	\caption{$\kappa\lambda$-PI}
	\label{alg:klPI}
	\begin{algorithmic}
		\STATE {\bfseries Initialize:} $\kappa \in [0,1],~\lambda\in[\kappa,1],~v \in \mathbb{R}^{|\S|}$
		\WHILE{some stopping condition is not satisfied}
		\STATE  $\pi\gets\arg\max\limits_{\pi'}\condE{\bigcdot}^{\pi'}\!\!\left[\sum_{t=0}^{\infty}(\kappa\gamma)^t (r_t+\gamma(1-\kappa) v_{t+1})\right]$
		\STATE  $v \gets \condE{\bigcdot}^{\pi}\left[\sum_{t=0}^{\infty}(\lambda\gamma)^t (r_t+\gamma(1-\lambda) v_{t+1})\right]$
		\ENDWHILE
		\STATE {\bfseries Return $\pi,v$}
	\end{algorithmic}
\end{algorithm}

Consider $\kappa\lambda$-PI (Algorithm~\ref{alg:klPI}, in which again the assignments hold component-wise). Its improvement stage is similar to the $\kappa$-PI algorithm; however, in the evaluation step, we apply $T^{\pi}_{\bar{\lambda},\kappa}$, or, equivalently, $T^{\pi}_\lambda$. Indeed, by setting $\lambda=\kappa$ we recover $\kappa$-VI (see Appendix~\ref{supp: kappa_VI_and_kappa_lambda_PI}) and by setting $\lambda=1$ we recover $\kappa$-PI. Moreover, by setting $\kappa=0$, we obtain the class of $\lambda$-PI algorithms. Notice that $\lambda\in[\kappa,1]$, whereas for $\lambda$-PI, $\lambda\in[0,1]$. We leave for future work the question 
whether the $\kappa\lambda$-PI makes sense for $\lambda\in[0,\kappa)$. 

At this point of the paper, we have reached a very general algorithmic formulation. We shall not only prove convergence for the $\kappa\lambda$-PI, but also provide a sensitivity analysis that shows how errors may propagate along the steps. This may indeed be of interest if we use approximations when computing a $\kappa$-greedy policy or updating the value function. Also, doing so, we generalize similar results on $\lambda$-PI \cite{alpi}.  In the following the subscript notation, $k$, refers to the kth iteration of the algorithm (the proof is deferred to the end of the paper for clarity, in Appendix~\ref{proofadp}). 

\begin{theorem}\label{adp}
  Let $\kappa\in[0,1]$ and $\lambda\in[\kappa,1]$. Assume that in Algorithm~\ref{alg:klPI} we employ noisy versions of the two steps at the kth iteration
  \begin{align}
    \pi_{k+1} &\leftarrow \Ga{\delta_{k+1}}{\kappa}(v_k) \label{eq_noisy_improvment}\\
  v_{k+1} &\leftarrow T_{\lambda}^{\pi_{k+1}} v_k + \epsilon_{k+1} \nonumber,
  \end{align}
  where the noisy improvement of \eqref{eq_noisy_improvment} means:
  \begin{align*}
T^{\pi_{k+1}}_\kappa v_k \ge T_\kappa v_k - \delta_{k+1}.
  \end{align*}
  Assume that for all  $k$, $\|\epsilon_k\|_\infty \le \epsilon$ and $\|\delta_k\|_\infty \le \delta$. Then,
  \begin{align*}
 \lim\sup_{k \to \infty} \|v^* - v^{\pi_k} \| &\le\! \frac{2 \xi \epsilon + \delta}{(1-\xi)^2}  \\
											 &\!=\frac{2\gamma(1-\kappa)(1-\kappa\gamma)\epsilon \! +\! (1-\kappa \gamma)^2 \delta}{(1-\gamma)^2}.  
  \end{align*}
\end{theorem}

Once again, we can measure how increasing $\kappa$ allows improving these asymptotic bounds. Furthermore, observe that the bounds do not depend on $\lambda$. 

By setting $\epsilon=\delta=0$ we get the following corollary. 
\begin{corollary}
The $\kappa\lambda$-PI algorithms converges to the optimal value-policy pair. 
\end{corollary}

\section{Relation to Existing Works}\label{interpretation_and_relation}
In this section we compare the $\kappa$-greedy policy (Section \ref{sec: kappa-step greedy}) to the $h$-greedy policy (Section \ref{sec: h-step greedy}). Furthermore, we connect previous works of the literature to the framework developed in this paper.

Let $v:\mathcal{S}\rightarrow \mathbb{R}$ be a value function, $h\in \mathbb{N} \textbackslash \{0\} $ and define the following random variable, the future $h$-step return:
\begin{align*}
r_{v}^{(h)}\eqdef \sum_{t=0}^{h-1}\gamma^t r_t+\gamma^h v(s_h).
\end{align*} 
Consider the $h$-greedy policy w.r.t. to $v$, defined in~\eqref{eq_h_greedy_def}. We have for all $\pi \in \G_h(v)$ and $s \in \mathcal{S}$,
\begin{align*}
\pi(s)= \arg\max_{\pi_0} \max_{\pi_1,..,\pi_{h-1}} \condE{s}^{\pi_0\dots\pi_{h-1}}\left[r_{v}^{(h)}\right].
\end{align*}
Alternatively, consider a different generalization of the greedy step in the form of a greedy policy w.r.t. a \emph{$\kappa$-weighted} average of the future rewards $\{r_{v}^{(h)}\}_{h\geq 1}$:
\begin{align*}
\G_{\mbox{\tiny $\kappa$-weighted}}(v) \eqdef \arg\max_{\pi'}\condE{\bigcdot}^{\pi'}[(1-\kappa)\sum_{h=0}^\infty \kappa^h r_{v}^{(h+1)}].
\end{align*} 
We can highlight the following strong relation between these two greedy sets (see Appendix~\ref{proof_proposition_kappa_greedy_policy_interpretation} in the supplementary material for a proof).
\begin{proposition}
\label{proposition_kappa_greedy_policy_interpretation}
Let $\delta_v(s_t)=r(s_t,\pi(s_t))+\gamma v(s_{t+1})-v(s_{t})$. We have
\begin{align*}
\G_{\text{$\kappa$-weighted}}(v)= \G_{\kappa}(v)=\arg\max_{\pi}\condE{\bigcdot}^{\pi}\! \! \left[\sum_{t=0}^\infty (\kappa\gamma)^t\delta_v(s_t) \right]\!\!.
\end{align*}

\end{proposition}
The second equality in Proposition~\ref{proposition_kappa_greedy_policy_interpretation} reveals a connection to two existing works: planning with shorter horizons \cite{ng2003shaping} and generalized advantage estimation \cite{schulman2015high}. 
Using the terminology of our work, the approach of \cite{ng2003shaping} is equivalent to performing a single $\kappa$-greedy improvement step. The theory we developped in this paper suggests that there is no reason to stop after only one step.
In \cite{schulman2015high}, the implemented algorithm is equivalent to an on-line Policy Gradient variant of the $\kappa$-PI algorithm. As Proposition \ref{proposition_kappa_greedy_policy_interpretation} in our work states, the objective function considered in \cite{schulman2015high} describes the very same surrogate MDP that is being solved in the improvement step of $\kappa$-PI. Moreover, in that work, an evaluation algorithm estimates the value of the current policy similarly to the evaluation stage of $\kappa$-PI. We can interpret the empirically demonstrated trade-off in $\lambda$ of \cite{schulman2015high} as a trade-off in $\kappa$. Finally, the policy update phase in the MCTS approach in RL, and in  Alpha-Go \cite{silver2016mastering,silver2017mastering,silver2017mastering2} as an instance of it, is conceptually similar to the policy update in an asynchronous online version of $h$-PI.

\section{Experimental Results}
\label{sec: experiments}
In this section we empirically test the $h$- and $\kappa$-PI algorithms on a toy grid-world problem. As mentioned before, in $h$-PI, performing the greedy step amounts to solving a $h$-horizon optimal control problem (Remark~\ref{astar}), and in $\kappa$-PI, it amounts to solving a $\gamma\kappa$-discounted stationary MDP (Remark~\ref{solving_with_smaller_discount}). In both cases, conducting these operations in practice can be done with either a generic planning, model-free or model-based algorithm. Here, we implement the $h$- and $\kappa$-greedy step via the VI algorithm. In the former case, we simply do $h$ steps, while in the latter case, we stop VI when the value change in max norm is less than $\epsilon=10^{-5}$ (other values did produce the same qualitative behavior). With this choice, note that $\kappa=1$ is equivalent to solving the problem with~VI.

\begin{figure*}[ht]
\centering
\subfloat{\includegraphics[scale=0.4]{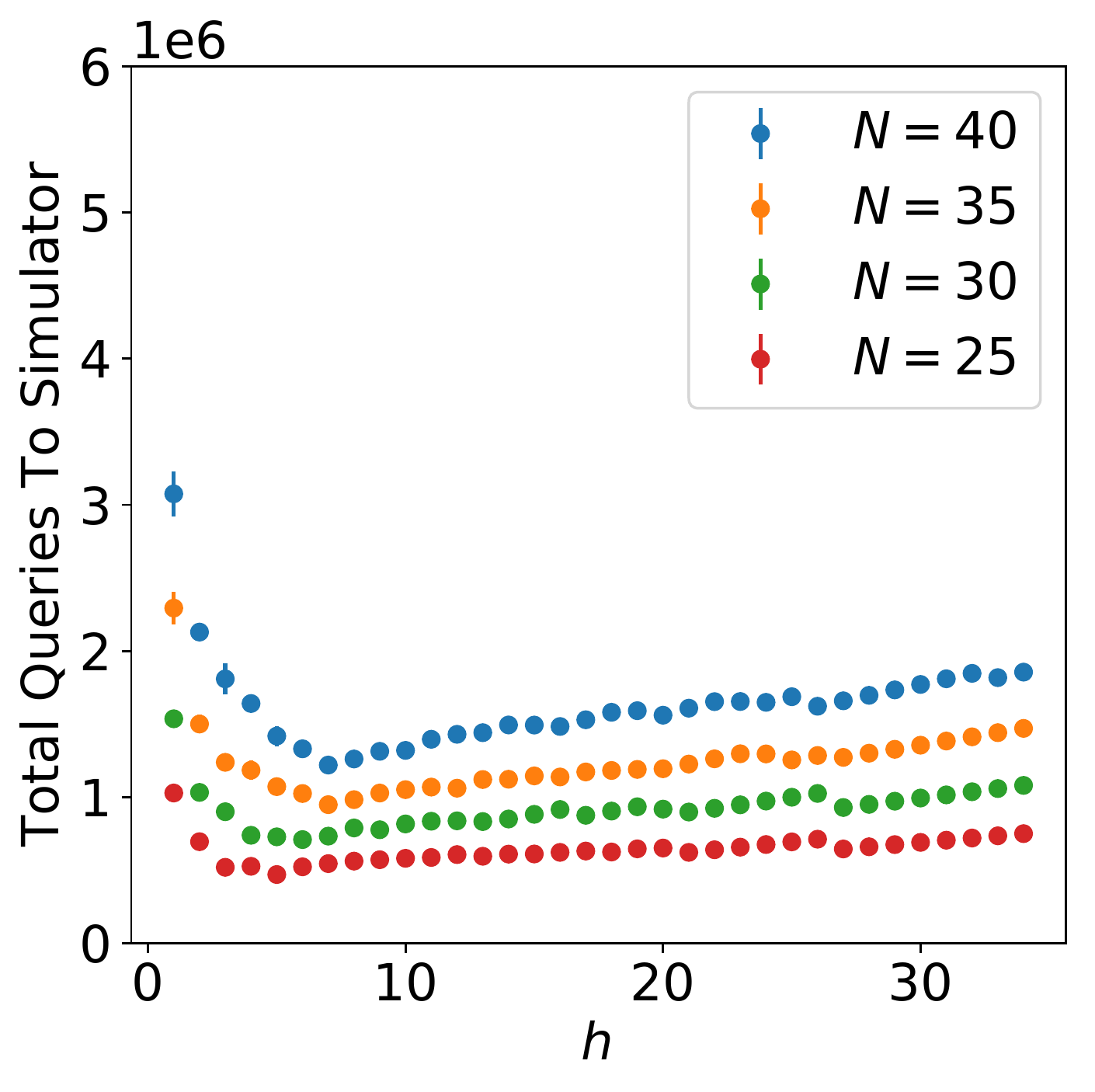}}
\centering
\subfloat{\includegraphics[scale=0.4]{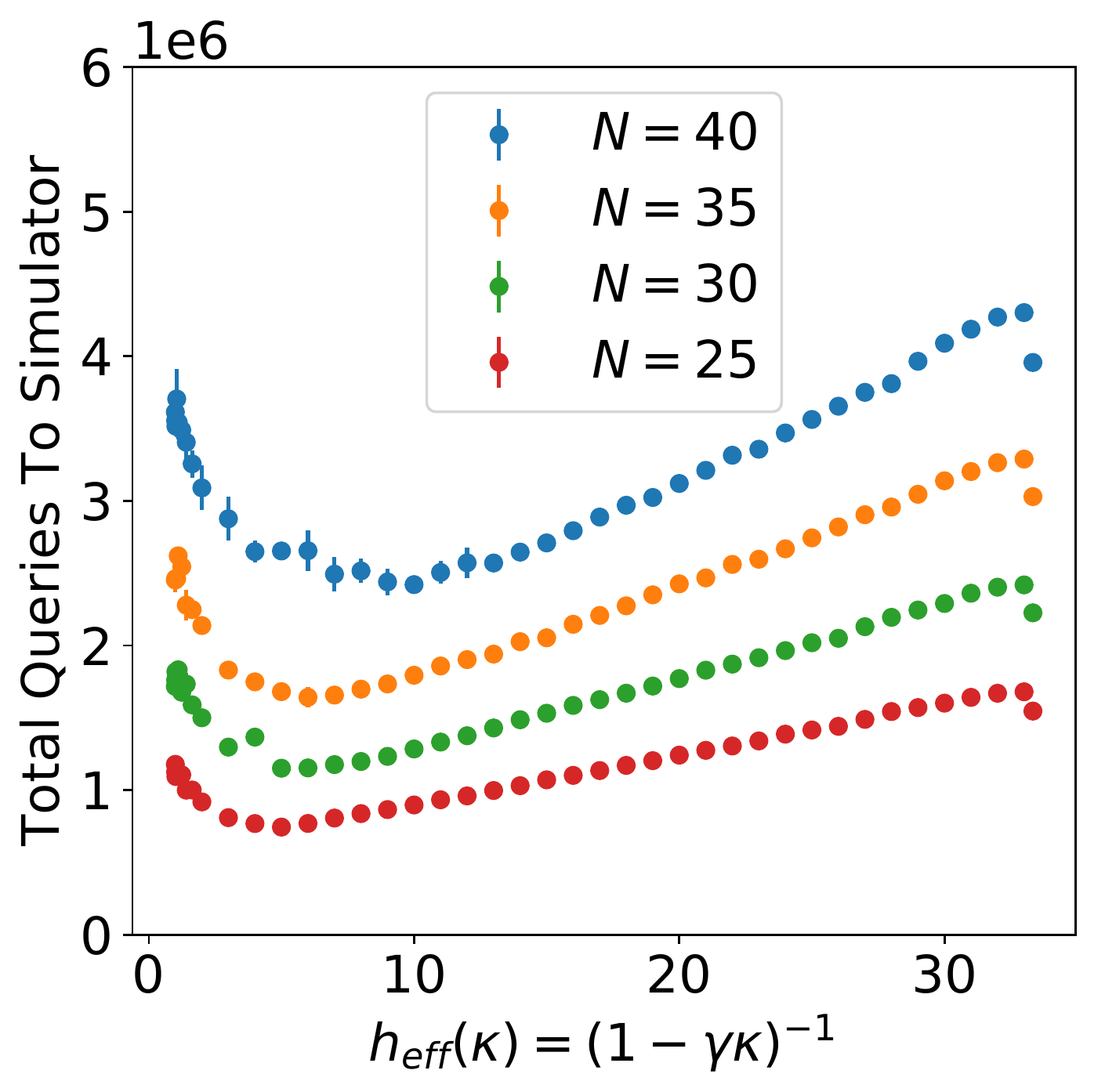}}
\centering
\subfloat{\includegraphics[scale=0.4]{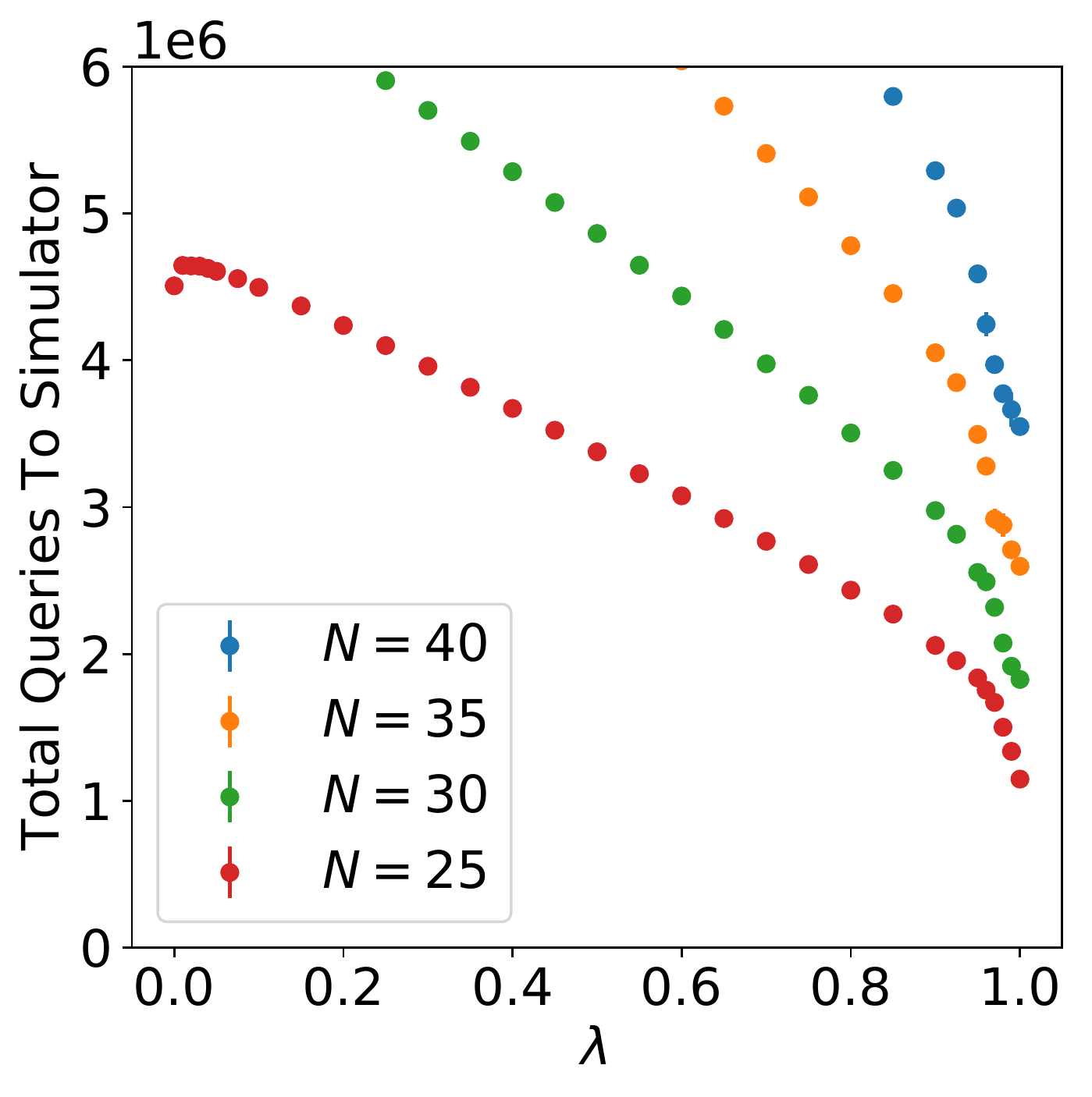}}
\caption{Empirical performance of $h$-PI, $\kappa$-PI and $\lambda$-PI for different grid sizes, $N$ (see Section~\ref{sec: experiments}). The shown results are the average of 5 experiments. The standard deviation is shown as errorbars. In all plots, the y-axis is the total queries to simulator until convergence, which we chose as a performance criterion. 
	 \textbf{(Left)} Performance of $h$-PI as a function of $h$. \textbf{(Center)} Performance of $\kappa$-PI as a function of $h_{\rm{eff}}(\kappa)=(1-\gamma\kappa)^{-1}$, the `effective' planning horizon. The $\kappa$ values that resulted in the lowest number of simulator queries are $\kappa_{\rm{opt}}=0.82,0.82,0.88,0.92$ for $N=25,30,35,40,$ respectively. \textbf{(Right)} Performance of $\lambda$-PI  as a function of $\lambda.$ This corresponds to $\kappa\lambda$-PI with $\kappa=0$.
	The three plots demonstrate that the algorithms introduced in this work,  $h$-PI and $\kappa$-PI, can outperform $\lambda$-PI in terms of best empirical performance.} 

\label{fig_emprical_results}
\end{figure*}

We conduct our simulations on a simple $N \times N$ deterministic grid-world problem with $\gamma=0.97$. The actions set is \{`up',`down',`right',`left',`stay'\}. In each experiment, we randomly chose a single state and placed a reward $r_g=1$. In all other states the reward was drawn uniformly from $[-0.1r_g,0.1r_g]$. In the considered problem there is no terminal state. Also, the entries of the initial value function are drawn from $\mathcal{N}(0,r^2_g)$. We ran $h$- and $\kappa$-PI and counted the \emph{total} number of calls to the simulator. Each such ``call'' takes a state-action pair $(s,a)$ as input, and returns the current reward and next (deterministic) state.

Figure \ref{fig_emprical_results} shows the average number of calls to the simulator as a function of $h$ or $\kappa$. For each value, experiments were conducted $5$ times. The figure depicts that optimal computational efficiency is obtained for some value of the parameters that is not trivial ($h \not\in \{1,\infty\}$ and $\kappa \not\in \{0,1\}$). Empirically, these 'optimal' parameter values slowly grow with the grid dimension $N$.
For a comparison, we measured the empirical performance of $\lambda$-PI (notice that $\lambda=1$, which is PI, corresponds to $\kappa$-PI with $\kappa=0$). Our simulations show that the performance of $\lambda$-PI is inferior to that of our new algorithms.



\section{Discussion and Future Work}

In this work, we introduced and formulated two possible approaches that generalize the traditional 1-step greedy operator.
Borrowing from the general principle behind Dynamic Programming, we  proposed a new family of techniques for solving a single complex problem via iteratively solving smaller sub-problems.  
We showed that the discussed approaches are coherent with previous lines of work, by generalizing existing analyses on the 1-step greedy policy to the $h$- and $\kappa$-greedy policies. In particular, we derived new algorithms and showed their convergence. By introducing and analyzing the $\kappa\lambda$-PI, we demonstrated that the $\kappa$-PI can be used with a `relaxed' value estimation and that the effects of noise are controlled. Last but not least, by making connections with some recent empirical works \cite{schulman2015high,silver2017mastering}, our work sheds some light on the reasons of their impressive success.
We conducted simulations on a toy example and have shown how a generic RL algorithm (VI) can be used for the greedy step of the $\kappa$- and $h$-PI frameworks. We empirically demonstrated that such a new algorithm leads to a performance improvement. By using other techniques for solving the greedy step (see in particular Remarks~\ref{astar} and \ref{solving_with_smaller_discount}), many new algorithms may be built. 

It may be interesting to consider online versions (i.e. stochastic approximations) of the algorithmic schemes we have introduced here, which would probably require to consider at least two time scales (one for the inner surrogate problems, one for the main loop). On the theoretical side, our first attempts in this direction suggest that previous approaches for proving online convergence of PI \cite{konda1999actor,kakade:02} may not be so straightforwardly generalized. We are currently working on this extension.

A potential practical extension of this work would be to consider a state-dependent $\kappa$ value, that is a function  $\kappa:\mathcal{S} \rightarrow [0,1]$. Though the details of such a generalization need to be written carefully, we believe that the machinery we developped here will still hold. We expect the convergence to be assured with rates that would intricately depend on this $\kappa$ function. Using such an approach, the algorithm designer could put more `prior knowledge' into the learning phase. In general though, understanding better when the choice of a $\kappa$ function would be good or not (and even understanding how to choose the $\kappa$ or $h$ parameters of the algorithms described here) is intriguing and deserves future investigation.




\appendix

\section{Proof of Theorem~\ref{adp}}
\label{proofadp}

We might follow the steps in \cite{opi,ampi}. But we give here an alternative and shorter proof, which is new even in the specific context of $\kappa=0$.

  Using (\ref{def2}), for any $\pi$, we have
  \begin{align*}
  T_\kappa^\pi v_1 - T^\pi_\kappa v_2 &= \gamma(1-\kappa)(I-\kappa\gamma P^\pi)^{-1} (v_1-   v_2)\\
					                  &= \xi D^\pi_\kappa P^\pi( v_1-v_2)\eqdef \xi \tilde P^\pi( v_1-v_2),
  \end{align*}
  where we defined $D^\pi_\kappa = (1-\kappa\gamma)(I-\kappa\gamma P^\pi)^{-1}$, a stochastic matrix, and $\tilde P_\kappa^\pi =D^\pi_\kappa P^\pi$, also a stochastic matrix.

  Define the following alternative error $\epsilon'_k=\epsilon_k - C_k e$, where $C_k=\frac{\max \delta_{k+1} + \max \epsilon_k-\xi \min \epsilon_k}{1-\xi}$ and $e$ is the constant vector made of ones. Since for all $\alpha$, $T_\kappa^\pi (v+\alpha e)=T_\kappa^\pi v + \xi \alpha e$ and $\Ga{\delta}{\kappa}(v)=\Ga{\delta}{\kappa}(v+\alpha e)$, the sequence of policies that can be generated by the original algorithm, with error $\epsilon_k$, is the same as that that would be generated by the algorithm with erros $\epsilon'_k$. From now on, let us consider the latter.

  With a similar invariance argument, let us assume, without loss of generality, that $v_0-T_\kappa^{\pi_1} v_0 \le 0$. Then, one can see that for all $k \ge 1$,
  \begin{align*}
    b_k &\eqdef v_k - T_\kappa^{\pi_{k+1}} v_k  \\
    &\le v_k - T_\kappa^{\pi_k} v_k + \delta_{k+1} \\
    & =  (v_k - \epsilon'_k) - T_\kappa^{\pi_k} (v_k - \epsilon'_k) + (1-\xi)\epsilon'_k + \delta_{k+1} \\
    & \le  T_\lambda ^{\pi_k} v_{k-1} - T_\kappa^{\pi_k}T_\lambda ^{\pi_k}v_{k-1} \\
    &\ \ \   + (\max \delta_{k+1} +\max \epsilon'_k-\xi \min \epsilon'_k)e  \\
    & =  T_\lambda ^{\pi_k} v_{k-1} - T_\lambda ^{\pi_k} T_\kappa^{\pi_k}v_{k-1} \\
    &\ \ \ + (\underbrace{\max \delta_{k+1} + \max \epsilon_k - C_k - \xi (\min \epsilon_k - C_k) }_{0})e \\
    & = \xi_\lambda \tilde P_\lambda^{\pi_k}(v_{k-1}-T_\kappa^{\pi_k}v_{k-1}) = \xi_\lambda \tilde P_\lambda^{\pi_k} b_{k-1}.
    \end{align*}
In the fifth relation we used Proposition \ref{proposition: Tlambda commutes} (see Appendix \ref{supp: affinity}) that shows $T_\kappa T_\lambda = T_\lambda T_\kappa$, and defined ${\xi_\lambda \eqdef \gamma\frac{1-\lambda}{1-\lambda\gamma}}$. Thus, since $\tilde P_\lambda^{\pi_k}$ has only non-negative elements, by induction, we have $b_k \le 0$ for all $k$.

Then, since $(1-\bar{\lambda})\sum_{j=0}^\infty \bar{\lambda}^j=1$, one can see that
\begin{align*}
  d_{k+1} &\eqdef v_* - (v_{k+1}-\epsilon'_{k+1})\\
   & = v^* - (1-\bar{\lambda})\sum_{j=0}^\infty \bar{\lambda}^j (T_\kappa^{\pi_{k+1}})^{j+1} v_k  \\
   & = (1-\bar{\lambda})\sum_{j=0}^\infty \bar{\lambda}^j (v^* - (T_\kappa^{\pi_{k+1}})^{j+1} v_k) .
\end{align*}
Each term in the sum satisfies:
\begin{align*}
  &v^* - (T_\kappa^{\pi_{k+1}})^{j+1} v_k & \\ 
  & = T_\kappa^{\pi^*} v^* - T_\kappa^{\pi^*}v_k + T_\kappa^{\pi^*}v_k \\
  &\ \ \ -  T_\kappa^{\pi_{k+1}}v_k + \sum_{i=1}^{j} (T_\kappa^{\pi_{k+1}})^i v_k - (T_\kappa^{\pi_{k+1}})^{i+1} v_k \\
  & ~\le~ \xi  \tilde P_\kappa^{\pi^*} (v^*-v_k)+ \delta_{k+1} + \sum_{i=1}^{j} (\xi \tilde P_\kappa^{\pi_{k+1}})^{i}b_k\\
  & ~\le~  \xi  \tilde P_\kappa^{\pi^*} (v^*-v_k) +\delta_{k+1},
\end{align*}
and hence we deduce that
\begin{align*}
 d_{k+1} & ~\le~  \xi \tilde P_\kappa^{\pi^*} (v^*-v_k) + \delta_{k+1} \\
 & =  \xi \tilde P_\kappa^{\pi^*} d_k - \xi \tilde P_\kappa^{\pi^*} \epsilon'_k + \delta_{k+1} 
\end{align*}
and thus
\begin{align*}
 &\max d_{k+1} \\
  & \le  \xi \max d_k  + \max \delta_{k+1}- \xi \min \epsilon'_k \\
  & =   \xi \max d_k + \max \delta_{k+1} \\
  &\ \ \      - \xi \left(\min \epsilon_k -  \frac{\max \delta_{k+1} + \max \epsilon_k-\xi \min \epsilon_k}{1-\xi} \right) \\
 & =  \xi \max d_k + \frac{\max \delta_{k+1}+ \xi (\max \epsilon_k- \min \epsilon_k)}{1-\xi},
\end{align*}
which eventually leads to
\begin{align*}
\lim_{k \to \infty} d_k \le \frac{\xi(\max \epsilon_k- \min \epsilon_k)+\max \delta_{k+1}}{(1-\xi)^2} .
\end{align*}
To conclude, we finally observe that
\begin{align*}
  s_k &\eqdef v_{k+1}-\epsilon_{k+1}-v_{\pi^{k+1}}\\ 
  & =(1-\bar{\lambda}) \sum_{j=0}^\infty \bar{\lambda}^j (T_\kappa^{\pi_{k+1}})^{j+1} v_k - v^{\pi_{k+1}} \\
  & = (1-\bar{\lambda}) \sum_{j=0}^\infty \bar{\lambda}^j  ((T_\kappa^{\pi_{k+1}})^{j+1} v_k - v^{\pi_{k+1}}),
\end{align*}
and each term of the sum satisfies:
\begin{align*}
  &(T_\kappa^{\pi_{k+1}})^{j+1} v_k - v^{\pi_{k+1}}\\
  & = \sum_{i \ge j+1} (T_\kappa^{\pi_{k+1}})^{i} v_k - (T_\kappa^{\pi_{k+1}})^{i+1} v_k \\
  & = \sum_{i \ge j+1} (\xi \tilde P_\kappa^{\pi_{k+1}})^{i}b_k  \le 0.
  \end{align*}
In other words, we deduce that $s_k$ is non-positive. The result follows from the fact that $v_*-v_{\pi_k}=d_k+s_k$ and that $\max \epsilon_k-\min \epsilon_k \le 2 \|\epsilon_k\|$.

\section*{Acknowledgments}
We thank Timothy Mann and Guy Tennenholtz for helpful discussions and the reviewers for their comments. We also thank Liyu Chen for pointing to an incomplete proof which is now fixed. This work was partially funded by the Israel Science Foundation under contract 1380/16 and by the European Community’s Seventh Framework Programme (FP7/2007-2013) under grant agreement 306638 (SUPREL)


\bibliography{kPI}
\bibliographystyle{icml2018}

\onecolumn

\section{Affinity of the Fixed Policy Bellman Operator and Consequences}\label{supp: affinity}

In this section we show that due to the affinity of the fixed-policy Bellman operator, $T^\pi$, it also preserves barycenters. We discuss the consequences of this observation and specifically prove that $T_{\lambda}$ commutes with any other $T_{\lambda'}$ operator.

\begin{lemma} \label{lemma: left distributive}
Let $\{v_i,\lambda_i\}_{i=0}^\infty$ be a series of value functions, $v_i\in \mathbb{R}^{|\mathcal{S}|}$, and positive real numbers, $\lambda_i\in \mathbb{R}^+$, such that $\sum_{i=0}^\infty \lambda_i=1$. Let $T^\pi$ be a fixed policy Bellman operator and $n\in \mathbb{N}$. Then,
\begin{align*}
&T^\pi(\sum_{i=0}^\infty \lambda_i v_i )= \sum_{i=0}^\infty \lambda_i T^\pi v_i,\\
&(T^\pi)^n(\sum_{i=0}^\infty \lambda_i v_i )= \sum_{i=0}^\infty \lambda_i (T^\pi)^n v_i.
\end{align*}
\end{lemma}

\begin{proof}
Using simple algebra and the definition of $T^\pi$ (see Definition \ref{def: Tpi}) we have that
\begin{align*}
T^\pi(\sum_{i=0}^\infty \lambda_i v_i ) &= r^\pi +\gamma P^\pi (\sum_{i=0}^\infty \lambda_i v_i)\\
&=r^\pi +\sum_{i=0}^\infty \lambda_i \gamma P^\pi  v_i\\
&=\sum_{i=0}^\infty \lambda_i  \left(r^\pi + \gamma P^\pi v_i \right)=\sum_{i=0}^\infty \lambda_i  T^\pi v_i.
\end{align*}

The second claim is a result of the first claim and is proved by iteratively applying the first relation.
\end{proof}

This lemma induces the following property.
\begin{proposition}\label{proposition: Tlambda commutes}
Let $T^\pi_{\lambda}$ be the fixed-policy $\lambda$ Bellman operator, $\lambda_1,\lambda_2 \in [0,1]$ and $v\in \mathbb{R}^{|\mathcal{S}|}$ be a value function. Then,
\begin{align*}
T^\pi_{\lambda_1} T^\pi_{\lambda_2} v = T^\pi_{\lambda_2} T^\pi_{\lambda_1} v,
\end{align*}
i.e, $T^\pi_{\lambda_1}, T^\pi_{\lambda_2}$ commute.
\end{proposition}

\begin{proof}
We have that
\begin{align*}
T^\pi_{\lambda_1} T^\pi_{\lambda_2} v & = (1-\lambda_1) \sum_{i=0}^\infty \lambda_1^i (T^\pi)^{i+1}  T^\pi_{\lambda_2} v \\
& = (1-\lambda_1) \sum_{i=0}^\infty \lambda_1^i (T^\pi)^{i+1} \left( (1-\lambda_2) \sum_{j=0}^\infty \lambda_2^j (T^\pi)^{j+1} v \right) \\
& = (1-\lambda_1)(1-\lambda_2) \sum_{i,j=0}^\infty \lambda_1^i\lambda_2^j (T^\pi)^{i+1}(T^\pi)^{j+1} v \\
& = (1-\lambda_2)(1-\lambda_1) \sum_{i,j=0}^\infty \lambda_2^j \lambda_1^i (T^\pi)^{j+1} (T^\pi)^{i+1} v \\
& = (1-\lambda_2) \sum_{j=0}^\infty \lambda_2^j (T^\pi)^{j+1}  \left( (1-\lambda_1) \sum_{i=0}^\infty \lambda_1^i (T^\pi)^{i+1} v \right) = T^\pi_{\lambda_2} T^\pi_{\lambda_1} v.
\end{align*}
The first and third relations use Definition \ref{def1}, the forth and sixth relations are due to Lemma \ref{lemma: left distributive}, and the fifth relation is due to the fact that every operator commutes with itself.
\end{proof}

\section{Proof of Lemma \ref{lemma_contracting_sequence_h_PI} }\label{proof_lemma_contracting_sequence_h_PI}
The proof goes as follows.
\begin{align*}
v^*-v^{\pi_k} &=T^{\pi^*}T^{h-1}v^*-T^{\pi^*}T^{h-1}v^{\pi_{k-1}}+T^{\pi^*}T^{h-1}v^{\pi_{k-1}} - T^{\pi_k}T^{h-1}v^{\pi_{k-1}}+T^{\pi_k}T^{h-1}v^{\pi_{k-1}}-T^{\pi_k}v^{\pi_k}\\
&\leq T^{\pi^*}T^{h-1}v^*-T^{\pi^*}T^{h-1}v^{\pi_{k-1}}+T^{\pi_k}T^{h-1}v^{\pi_{k-1}}-T^{\pi_k}v^{\pi_k}\\
&\leq \gamma P^{\pi^*}(T^{h-1}v^*-T^{h-1}v^{\pi_{k-1}}) + \gamma P^{\pi_k}(T^{h-1}v^{\pi_{k-1}}-v^{\pi_{k}})\\
&\leq \gamma P^{\pi^*}(T^{h-1}v^*-T^{h-1}v^{\pi_{k-1}}).
\end{align*}
The second relation uses that  $T^{\pi^*}T^{h-1}v^{\pi_{k-1}}\leq T^{\pi_k}T^{h-1}v^{\pi_{k-1}}$ since $\pi_k\in \G_h(v^{\pi_{k-1}})$. The third relation holds by definition, and the fourth because $T^{h-1}v^{\pi_{k-1}}\leq v^{\pi_{k}}$, as as seen from Lemma \ref{lemma_hPI_improves}. Since $v^*-v^{\pi_k}$ is non-negative we can take the max norm. By using the fact that $T$ is a $\gamma$ contraction in the max norm, $T^{h-1}$ is a $\gamma^{h-1}$ contraction in the max norm. Thus, by taking the max norm of both sides we prove the claim.

\section{Proof of Lemma \ref{lemma_contracting_sequence_kPI} }\label{proof_lemma_contracting_sequence_kPI}
Using the fixed point property of the operator in Lemma \ref{gmdp},
\begin{align*}
v^*-v^{\pi_k} &=T^{\pi^*}_\kappa v^*-T^{\pi^*}_\kappa v^{\pi_{k-1}}+T^{\pi^*}_\kappa v^{\pi_{k-1}}-T^{\pi_k}_\kappa v^{\pi_{k-1}}+T^{\pi_k}_\kappa v^{\pi_{k-1}}-v^{\pi_k}\\
&\leq T^{\pi^*}_\kappa v^*-T^{\pi^*}_\kappa v^{\pi_{k-1}}+T^{\pi_k}_\kappa v^{\pi_{k-1}}-v^{\pi_k}\\
&\leq T^{\pi^*}_\kappa v^*-T^{\pi^*}_\kappa v^{\pi_{k-1}}.
\end{align*}
The second relation holds since $T^{\pi^*}_\kappa v^{\pi_{k-1}}\leq T^{\pi_k}_\kappa v^{\pi_{k-1}}=T_\kappa v^{\pi_{k-1}},$ as follows from the definition of $T_\kappa$ (see \eqref{prop1}). The third relation uses that $T^{\pi_k}_\kappa v^{\pi_{k-1}}\leq v^{\pi_k},$ as follows from the proof of Lemma \ref{lemma_kPI_improvement}. Since $v^*-v^{\pi_k}$ is positive, we take the max norm and use the fact that $T^{\pi^*}_\kappa$ is $\xi$ contraction in the max norm to conclude the proof.

\section{Proof of Lemma \ref{lemma_T_pi_kappa_lambda_is_T_lambda} }\label{proof_lemma_T_pi_kappa_lambda_is_T_lambda}

Since for the operator $T^{\pi}$ the distributive property holds, i.e, $aT^\pi+b(T^\pi)^2 = (a+bT^\pi)T^\pi$, we have the following relation for any $\lambda\in[0,1]$.
\begin{align*}
(1-\lambda T^\pi)T_\lambda^\pi=&(1-\lambda T^\pi)(1-\lambda)\sum_{i=0}^\infty \lambda^i (T^\pi)^{i+1} \\
=&(1-\lambda)\sum_{i=0}^\infty \lambda^i (T^\pi)^{i+1}- (1-\lambda)\lambda T^\pi \sum_{i=0}^\infty \lambda^i (T^\pi)^{i+1}\\
=&(1-\lambda)\sum_{i=0}^\infty \lambda^i (T^\pi)^{i+1}- (1-\lambda)\sum_{i=1}^\infty \lambda^i (T^\pi)^{i+1}\\
=&(1-\lambda)T^{\pi}.
\end{align*}
The first relation uses that the infinite sum converges and the definition of the $T^\pi_\lambda$ operator (cf. \eqref{def1}), in the third relation we used Lemma \ref{lemma: left distributive}, and in the forth relation all terms cancel out except for the first term (with $i=0$) in the first series. 

We thus get \emph{a new relation} the $T^\pi_\lambda$ operator satisfies,
\begin{equation}\label{lambda_PI_relation}
(1-\lambda T^\pi)T_\lambda^\pi =(1-\lambda)T^\pi.
\end{equation} 

Let $\bar{\lambda}\in[0,1],\kappa\in[0,1]$ and $\lambda =\kappa + \bar{\lambda}(1-\kappa)$. We now expand $(1-\lambda T^\pi) T_{\bar{\lambda},\kappa}^\pi$.
\begin{align}
(1-\lambda T^\pi)T_{\bar{\lambda},\kappa}^\pi& = (1-\lambda T^\pi)(1-\bar{\lambda})\sum_{j=0}^\infty \bar{\lambda}^j (T_\kappa^{\pi})^{j+1} \nonumber \\
&= (\left(1-\kappa T^\pi\right) - \left(\bar{\lambda}(1-\kappa)T^\pi\right))(1-\bar{\lambda})\sum_{j=0}^\infty \bar{\lambda}^j (T_\kappa^{\pi})^{j+1} \nonumber \\
& = (1-\kappa T^\pi)\left((1-\bar{\lambda})\sum_{j=0}^\infty \bar{\lambda}^j (T_\kappa^{\pi})^{j+1}\right)-\bar{\lambda}(1-\kappa)T^\pi\left( (1-\bar{\lambda})\sum_{j=0}^\infty \bar{\lambda}^{j} (T_\kappa^{\pi})^{j+1} \right)\nonumber \\
& = (1-\bar{\lambda})\sum_{j=0}^\infty \bar{\lambda}^j (1-\kappa T^\pi)(T_\kappa^{\pi})^{j+1}-\bar{\lambda}(1-\kappa)T^\pi\left( (1-\bar{\lambda})\sum_{j=0}^\infty \bar{\lambda}^{j} (T_\kappa^{\pi})^{j+1} \right)\nonumber \\
& = \sum_{j=0}^\infty \bar{\lambda}^j (1-\bar{\lambda})(1-\kappa )T^\pi(T_\kappa^{\pi})^{j}-\bar{\lambda}(1-\kappa)T^\pi\left( (1-\bar{\lambda})\sum_{j=0}^\infty \bar{\lambda}^{j} (T_\kappa^{\pi})^{j+1} \right) \nonumber \\
& = (1-\kappa )T^\pi \underbrace{ \left((1-\bar{\lambda})\sum_{j=0}^\infty \bar{\lambda}^j (T_\kappa^{\pi})^{j}-(1-\bar{\lambda})\sum_{j=0}^\infty \bar{\lambda}^{j+1} (T_\kappa^{\pi})^{j+1}\right)}_{=(1-\bar{\lambda})} \nonumber  \\
&= (1-\bar{\lambda})(1-\kappa )T^\pi = (1-\lambda)T^\pi. \nonumber 
\end{align}

In the forth relation we used Lemma \ref{lemma: left distributive}, in the fifth relation we used \eqref{lambda_PI_relation} , with $\kappa$ instead of $\lambda$, and in the sixth relation we again used Lemma \ref{lemma: left distributive}. We thus get that the operators $T_{\bar{\lambda},\kappa}^\pi$ and $T^\pi_\lambda$ satisfy the following relations.
\begin{align*}
&(1-\lambda T^\pi)T_{\bar{\lambda},\kappa}^\pi = (1-\lambda)T^\pi,\\
&(1-\lambda T^\pi)T^\pi_\lambda = (1-\lambda)T^\pi.
\end{align*}

Subtracting the two equations and using $T^\pi = r^\pi(\cdot) +\gamma P^\pi$ we get\footnote{In this context, we treat $r^\pi(\cdot)$ as a constant mapping from $\mathbb{R}^{|\mathcal{S}|}$ to itself; i.e., for any $v\in\mathbb{R}^{|\mathcal{S}|},\ r^{\pi}(v)=r^\pi$.},
\begin{align*}
\left[1- \lambda (r^\pi + \gamma P^\pi)\right](T_{\bar{\lambda},\kappa}^\pi - T^\pi_\lambda)  = (1- \lambda \gamma P^\pi)(T_{\bar{\lambda},\kappa}^\pi - T^\pi_\lambda) = 0.
\end{align*}

Since $(1- \lambda \gamma P^\pi)$ is non-zero (i.e, exists $v\in \mathbb{R}^{|\mathcal{S}|}$ s.t $(1-\lambda P^\pi)v \neq 0$), we get that
\begin{align*}
T_{\bar{\lambda},\kappa}^\pi = T^\pi_\lambda,
\end{align*}

which concludes the proof.

\section{Equivalence of $\kappa$-VI and $\kappa\lambda$-PI with $\lambda=\kappa$}
\label{supp: kappa_VI_and_kappa_lambda_PI}

The equivalence of $\kappa\lambda$-PI when $\lambda=\kappa$ to $\kappa$-VI is similar to the equivalence of $\lambda$-PI to VI for $\lambda=0$ \cite{alpi}.  To see this, start by considering the $\kappa$-VI algorithm (see Section~\ref{sec: kappa_lambda_PI}). Let $v_k$ be the value of the algorithm at the $k$-th iteration. Then, the value at the next iteration, $v_{k+1}$, is
\begin{align}
\label{eq : kappa-VI}
v_{k+1} = T_\kappa v_{k}  =  \max_{\pi'} T^{\pi'}_\kappa v_k.
\end{align}
Now consider the update performed by the $\kappa\lambda$-PI algorithm with $\lambda=\kappa$ (see Algorithm \ref{alg:klPI}). By construction, it will apply the following update.
\begin{align}
&\pi_{k+1} = \arg\max_{\pi'} T^{\pi'}_\kappa v_k \label{eq : kl-VI_policy}\\
&v_{k+1} =  T^{\pi_{k+1}}_\kappa v_k. \label{eq : kl-VI_value}
\end{align}

Using these updating equations we see that the value at the $k+1$ iteration of this algorithm can be written as
\begin{align*}
v_{k+1} =  T^{\pi_{k+1}}_\kappa v_k = \max_{\pi'} T^{\pi'}_\kappa v_k,
\end{align*}
since $\pi_{k+1}$ is the policy attaining the maximum in \eqref{eq : kl-VI_policy}.  

Thus, both $\kappa$-VI and $\kappa\lambda$-PI with $\lambda=\kappa$ produce the same sequence of values, given equal initial value, $v_0$. In this sense, both of the algorithms are equivalent. However, we note that these algorithms are not \emph{computationally} equivalent. This can be seen by focusing on the computations steps per iteration. The $\kappa$-VI algorithm updates a value function in every iteration whereas $\kappa\lambda$-PI (with $\lambda=\kappa$) calculates a policy \eqref{eq : kl-VI_policy} and then updates the current value \eqref{eq : kl-VI_value}. These two forms do not necessarily share similar computational complexity.

\section{Proof of Proposition \ref{proposition_kappa_greedy_policy_interpretation}}\label{proof_proposition_kappa_greedy_policy_interpretation}
The proof is close in spirit to a related proof in \cite{schulman2015high} and can also be derived using the relations in \eqref{def1}, \eqref{def2} and \eqref{def4}. We explicitly give here the proof for completeness. Furthermore, in this proof we explicitly show that the greedy sets are equal, a notion developed in this work.

We prove the first equality and then the second equality of the proposition. The first equality follows by a direct algebraic manipulation on the random variables inside the mean. For clarity, we use the notation $a_t=\pi(s_t)$.
\begin{align}
&(1-\kappa)\sum_{h=0}^\infty \kappa^h r_{v}^{(h+1)}= \nonumber\\
(1-\kappa)(&r(s_0,a_0)+\gamma v(s_1)+\nonumber\\
&\kappa r(s_0,a_0)+\kappa \gamma r(s_1,a_1)+\kappa\gamma^2 v(s_2)+\nonumber\\
&\kappa^2 r(s_0,a_0)+\kappa^2 \gamma r(s_1,a_1)+\kappa^2 \gamma^2 r(s_2,a_2)+\kappa^2\gamma^3 v(s_3)+\nonumber\\
&\qquad \qquad \qquad \qquad \vdots \qquad \qquad \qquad \qquad)=\nonumber\\
&r(s_0,a_0)+(1-\kappa)\gamma v(s_1)+\nonumber\\
&\kappa \gamma r(s_1,a_1)+(1-\kappa)\kappa\gamma^2 v(s_2)+\nonumber\\
&\kappa^2 \gamma^2 r(s_2,a_2)+(1-\kappa)\kappa^2\gamma^3 v(s_3)+\nonumber\\
&\qquad \qquad \qquad \qquad \vdots \qquad \qquad \qquad \qquad=\nonumber\\
&r(s_0,a_0)+(1-\kappa)\gamma v(s_1)+\nonumber\\
&\kappa \gamma (r(s_1,a_1)+(1-\kappa)\gamma v(s_2))+\nonumber\\
&\kappa^2 \gamma^2 (r(s_2,a_2)+(1-\kappa)\gamma v(s_3))+\nonumber\\
&\qquad \qquad \qquad \qquad \vdots \qquad \qquad \qquad \qquad=\nonumber\\
&\sum_{t=0}^\infty (\kappa\gamma)^t(r(s_t,a_t)+(1-\kappa)\gamma v(s_{t+1})). \label{eq: kappa gamma sum}
\end{align}

In the second relation we used $\sum_{i=h}^\infty \kappa^i=\frac{\kappa^h}{1-\kappa}$, and in the last we packed the infinite sum and introduced a summation variable $t$. Since for every realization of the random variables this equality holds, it also holds for the mean. This concludes the proof of the first equality since we have  explicitly shown \emph{equality} between the optimization criteria.

We now prove the second equality, again with direct algebraic manipulation. We start with the last term in the proposition. Consider again the random variables inside the mean.

\begin{align}
\sum_{t=0}^\infty (\kappa\gamma)^t\delta_v(s_t)&=\sum_{t=0}^\infty (\kappa\gamma)^t(r(s_t,a_t)+\gamma v(s_{t+1})-v(s_t)) \nonumber\\
&=\sum_{t=0}^\infty (\kappa\gamma)^t(r(s_t,a_t)+\gamma v(s_{t+1})-\kappa\gamma v(s_{t+1})) -v(s_0)\nonumber\\
&=\sum_{t=0}^\infty (\kappa\gamma)^t(r(s_t,a_t)+\gamma(1-\kappa) v(s_{t+1})) -v(s_0) \label{eq: sum with v0},
\end{align}
where the second relation holds because $\sum_{t=0}^\infty (\kappa \gamma)^t v(s_t) = \sum_{t=0}^\infty (\kappa \gamma)^t \kappa \gamma v(s_{t+1}) + v(s_0)$.
Eq.~\eqref{eq: sum with v0} varies from \eqref{eq: kappa gamma sum} only by  $v(s_0)$. The solution of the optimization problem, \emph{finding the $\kappa$-greedy policy}, is invariant to the addition of $v(s_0)$, since it is a constant w.r.t. the optimization problem.  This concludes the proof.
\end{document}